\newcommand{\nop}[1]{}
\documentclass{sig-alternate}

\usepackage{graphicx} 
\usepackage{array} 
\usepackage{slashbox}
\usepackage{algorithm}
\usepackage{algpseudocode}
\usepackage{multirow}
\renewcommand{\min}[1]{\underset{#1}{\text{min}}}
\newcolumntype{C}[1]{>{\centering\arraybackslash}m{#1}}

\newdef{definition}{Definition}

\begin{document}
%
\conferenceinfo{e-Energy}{'14 Cambridge, UK}

\title{Plug and Play! A Simple, Universal Model for Energy Disaggregation}

\numberofauthors{3} 
\author{
\alignauthor Guoming Tang, Kui Wu\\
       \affaddr{Department of Computer Science}\\
       \affaddr{University of Victoria, Victoria, BC, Canada}\\
\alignauthor Jingsheng Lei\\
       \affaddr{School of Computer and Information Engineering}\\
       \affaddr{Shanghai University of Electric Power, China}\\
\alignauthor Jiuyang Tang\\
       \affaddr{Information Systems Engineering Lab}\\
       \affaddr{National University of Defense Technology, China}\\
}

\date{5 Jan. 2013}

\maketitle

\begin{abstract}
Energy disaggregation is to discover the energy consumption of individual appliances from their aggregated energy values. To solve the problem, most existing approaches rely on either appliances' signatures or their state transition patterns, both hard to obtain in practice. Aiming at developing a simple, universal model that works without depending on sophisticated machine learning techniques or auxiliary equipments, we make use of easily accessible knowledge of appliances and the sparsity of the switching events to design a Sparse Switching Event Recovering (SSER) method. By minimizing the total variation (TV) of the (sparse) event matrix, SSER can effectively recover the individual energy consumption values from the aggregated ones. To speed up the process, a Parallel Local Optimization Algorithm (PLOA) is proposed to solve the problem in active epochs of appliance activities in parallel. Using real-world trace data, we compare the performance of our method with that of the state-of-the-art solutions, including Least Square Estimation (LSE) and iterative Hidden Markov Model (HMM). The results show that our approach has an overall higher detection accuracy and a smaller overhead.   
\end{abstract}



\section{Introduction}
Energy disaggregation, also known as non-intrusive appliance load monitoring (NIALM), aims to learn the energy consumption of individual appliances from their aggregated energy consumption values, \textit{e.g.}, the total energy consumption of a house. With accurate energy disaggregation, the household can 1) learn how much energy each appliance consumes, 2) take necessary actions to save energy, and 3) participate in utility demand response programs. Furthermore, with smart meters broadly deployed in many countries, sufficiently high resolution of energy data can be collected, making it feasible to develop efficient energy disaggregation solutions.

Due to its critical meaning, the energy disaggregation problem has attracted more and more attention since 1980s. Recently, it has also drawn attention from both large electronics companies and small start-ups, such as Intel, Belkin, GetEmme, and Navetas. While many methods have been developed for energy disaggregation, according to~\cite{zeifman2011nonintrusive}, no solutions work well for all types of household appliances. They either work poorly for new types of appliances or require complex machine learning methods to learn appliances' (latent) features. 

\subsection{Motivations}

We are motivated to develop a simple and broadly applicable solution for energy disaggregation, based on the following observations:  

\begin{itemize}
\item Most existing methods are based on appliances' energy usage patterns, also called the \emph{signatures} of appliances, which are hard to obtain without particular machine learning techniques or auxiliary measurements. For example, in~\cite{gupta2010electrisense}, extra equipments are needed to detect the activities of appliances based on high frequency electromagnetic interference (EMI).

\item The \emph{rated power} of an appliance is normally available in practice, from users' manual, technical specification or public web sites such as~\cite{EnergyStar2013}.   \nop{Fig.~\ref{fig:userManual} shows the power consumption of a microwave, which is specified in the user's manual\nop{\footnote{Available at \url{http://www.panasonic.co.uk/html/en_GB/Products/Kitchen+Appliances/Microwave+Combination/NN-CT579SBPQ/Specification/4923503/}}}. 

\begin{figure}[!ht]
\centering
\includegraphics[width=3.2in,height=1.5in]{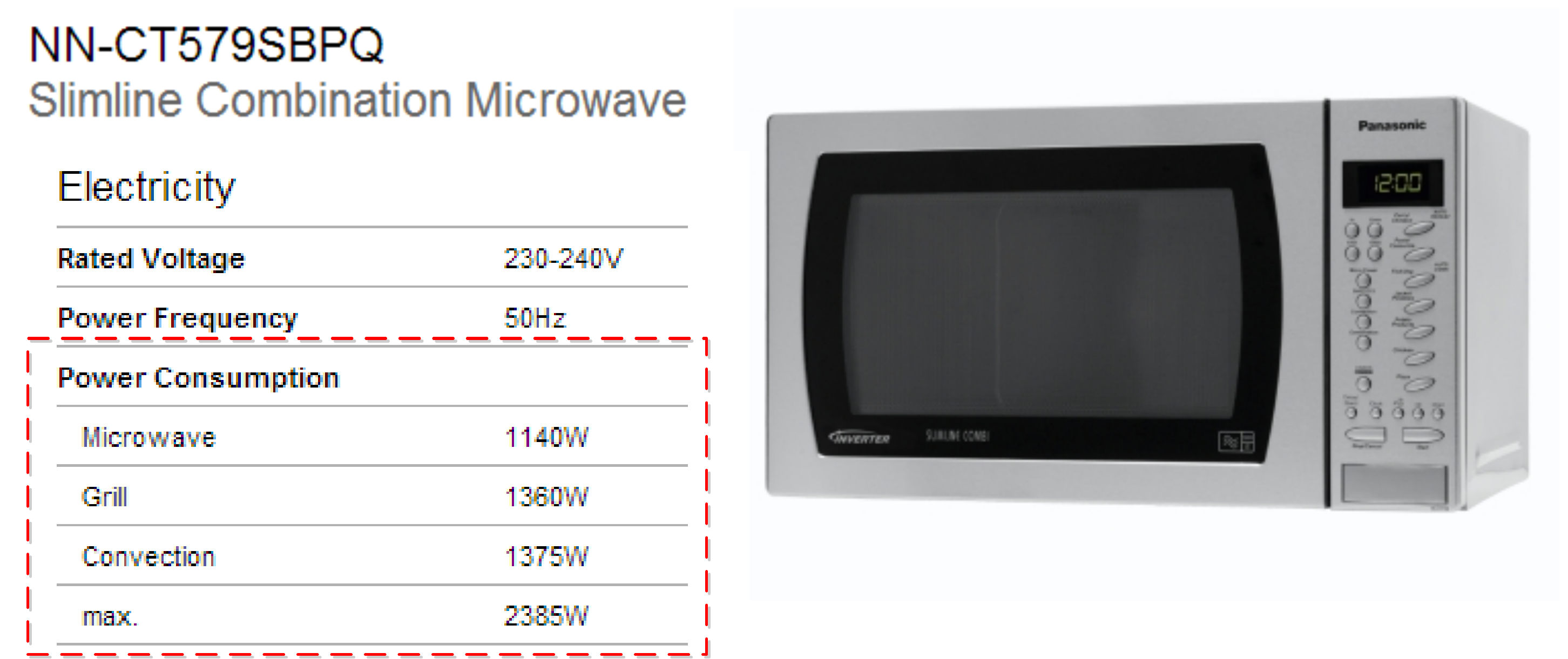}
\caption{Power consumption of a microwave specified in the user's manual}\label{fig:userManual}
\end{figure}
}
\item The temporal sparsity of the on/off switching events has been recognized as a general feature suitable for most appliances, \textit{i.e.}, an appliance cannot switch on/off many times in a very short time period. Nevertheless, this property has not been fully entrusted with an important post in energy disaggregation. To develop a universal solution for energy disaggregation, the sparsity feature could play an important role and should be regarded as a significant knowledge.

\end{itemize}

\subsection{Our Contributions}
Aiming at establishing an easy-to-use, universal model for energy disaggregation, we make the following contributions in this paper:
\begin{itemize}
\item We do not rely on appliances' \textit{signature}. Instead, we use the appliances' rated power and power deviation, which are easy to obtain, \textit{e.g.}, from the user's guide of appliances. With experimental evaluation, we show that the method is robust even if this information is not very accurate.

\item Based on the simple power model and the sparsity property of appliance activities, we establish a universal Sparse Switching Event Recovering (SSER) optimization model. Unlike existing methods that minimize the aggregated residual value, our method tries to minimize the total variation of on/off switching events. The new objective function, while very effective as we will show later, has \emph{never} been explored before to solve the energy disaggregation problem.

\item We develop a Parallel Local Optimization Algorithm (PLOA) to solve SSER, which can significantly reduce the computational complexity of the original problem and is guaranteed to obtain the optimal solution if some weak hypotheses hold. 

\item We build a small-scale energy monitoring platform for a group of household appliances, and evaluate our method using the real-world trace data collected over the platform. The experimental results indicate that our approach has an overall better performance than state-of-the-art solutions, including the well-known Least Square Estimation (LSE) method and a recently-developed machine learning method using iterative Hidden Markov Model (HMM).
\end{itemize}

\section{Related Work}

Tremendous research efforts have been devoted to solving the energy disaggregation problem. The existing approaches can be roughly divided into two categories: \emph{signature based} methods and \emph{state transition based} methods.

\subsection{Signature Based Methods}
Most approaches are  based on appliances' signatures, \textit{i.e.}, specific features such as the real/reactive power, current, and voltage of running appliances~\cite{hart1992nonintrusive}. These methods need the support of high sampling rate and build either steady or transient signal features of appliances with  labeled training datasets. The signal features are treated as the appliances' signatures~\cite{norford1996non,laughman2003power}, based on which event detection schemes are developed to detect appliances' on/off as well as different running states. The detected events are ascribed to certain appliances' activities via classification~\cite{dong2012event,suzuki2008nonintrusive,lam2007novel}. In addition to time-domain signal features,  spectral analysis has also been adopted to search for appliances' signatures in the frequency domain~\cite{leeb1995transient,shaw2008nonintrusive,gupta2010electrisense}.

The performance of signature based approaches depends greatly on the uniqueness of an appliance's signature. In practice, however, the signatures of different appliances may overlap with each other, causing inaccurate event detection. Even for the same type of appliances, it may be hard to obtain the widely acceptable signature~\cite{zeifman2011nonintrusive}. In other words, it is hard to generalize the signature learned from a particular device's operating data. Consequently, even though a method may have a good performance over a specific group of appliances, it may suffer in other datasets, caused by the over-fitting problem due to over-specific signatures. Due to these difficulties, it is not easy to use signature based methods for unambiguous appliance detection and classification.

\subsection{State Transition Based Methods}

A number of methods made use of state transition in appliances' activities for energy disaggregation. Recently, the Hidden Markov Model (HMM) was adopted to model the state transition patterns of appliances. The hidden states of each appliance at each time instant are predicted by inference algorithms, such as the Viterbi algorithm, with the observed emission probabilities~\cite{kim2012unsupervised,parson2012non}. Non-negative sparse coding was proposed to solve the energy disaggregation problem in~\cite{kolter2010energy}. It was further discussed in~\cite{figueiredo2013regularization}, in which a training process was needed to obtain the basis vector related to the state transition patterns of different appliances. Although some other works, such as~\cite{wang2012tracking}, were not to solve the energy disaggregation problem, they also utilized the appliance state transition information and their results may be helpful for energy disaggregation. 

The methods in this category usually need a large number of trainings, and thus are time consuming. In addition, the performance highly relies on the pattern of appliances' activities in the training datasets, and as such the performance may vary significantly from test to test. \newline

Finally, in the above two types of approaches, optimization algorithms were used to search for optimal solutions. Generally, the objective was to minimize the difference between the predicted value and real aggregation value. For example, in~\cite{suzuki2008nonintrusive}, Least Square Estimation (LSE) was used to find the tightest fit for the  aggregated waveform. Nevertheless, as we will disclose in this paper, such solutions usually do not match well the true switching events of appliances, leading to inaccurate energy disaggregation results.

\section{Sparsity of Switching Events}
For ease of reference, we list the notations in Table~\ref{tab:notation}. 

\begin{table}[!ht]
	\caption{Table of Notations}\label{tab:notation}
	\centering
		\begin{tabular}{m{1.2cm} m{6.4cm}}
		  \hline
			Symbol  &  Explanation\\
			\hline
			$S$ & state matrix \\
			$S^{(n)}$ & state vector of the $n$-th appliance along the timeline\\
			$S_{t}$ &  state vector of all appliances at time $t$\\
			$S^{(n)}_{t}$ & state of the $n$-th appliance at time $t$\\
			$S_{t:t+\ell}$ & states of all appliances from time $t$ to $t+\ell$\\
			$\Delta{S}$ & event matrix\\
			$\Delta{S}^{(n)}_{t}$ & switching event of the $n$-th appliance at time $t$\\
			$D$ & differential matrix \\
			$X$ & aggregated power vector\\
      $X^{(n)}$ & power readings of the $n$-th appliance along the timeline\\
      $X_{t}$ & aggregated power reading of all appliances at time $t$\\
      $X^{(n)}_{t}$ & power reading of the $n$-th appliance at time $t$\\
			$X_{t:t+\ell}$ & aggregated power readings of all appliances from time $t$ to $t+\ell$\\
			
			$I$ & stand-by power vector\\
      $I_{n}$ & stand-by power of the $n$-th appliance\\
			$P$ & rated power vector\\
      $P_{n}$ & rated power of the $n$-th appliance\\
			$P_{0}$ & baseline power of a house\\
      $\Theta$ & power deviation vector\\
      $\Theta_{n}$ & power deviation of the $n$-th appliance\\
			$W$ & set of active epochs\\
      $W_{k}$ & the $k$-th active epoch\\
			
			$\Gamma$ & set of power modes\\
      $\Gamma^{n}_{k}$ & the $k$-th power mode of the $n$-th appliance\\	
			\hline
		\end{tabular}
\end{table}

Fig.~\ref{fig:illustration_sparsityNbursty}\footnote{The figure is borrowed from~\cite{kolter2011redd} with slight modification for better illustration.} \nop{available at~\textit{http://redd.csail.mit.edu/}.} shows an example of energy consumption and appliances on/off switching events in a typical house during one day. From the figure, we can see that:
\begin{itemize}
\item As shown in Fig.~\ref{fig:illustration_sparsityNbursty}-a, the appliances do not switch on/off frequently in the whole time period.   
\item Most switching events happen in a small number of time intervals, which we call \emph{active epochs} (refer to Section~\ref{burstwindow} for formal definition) and are illustrated with shaded windows in Fig.~\ref{fig:illustration_sparsityNbursty}-b.
\end{itemize}

\begin{figure}[!ht]
\centering
\includegraphics[width=3.3in,height=2.8in]{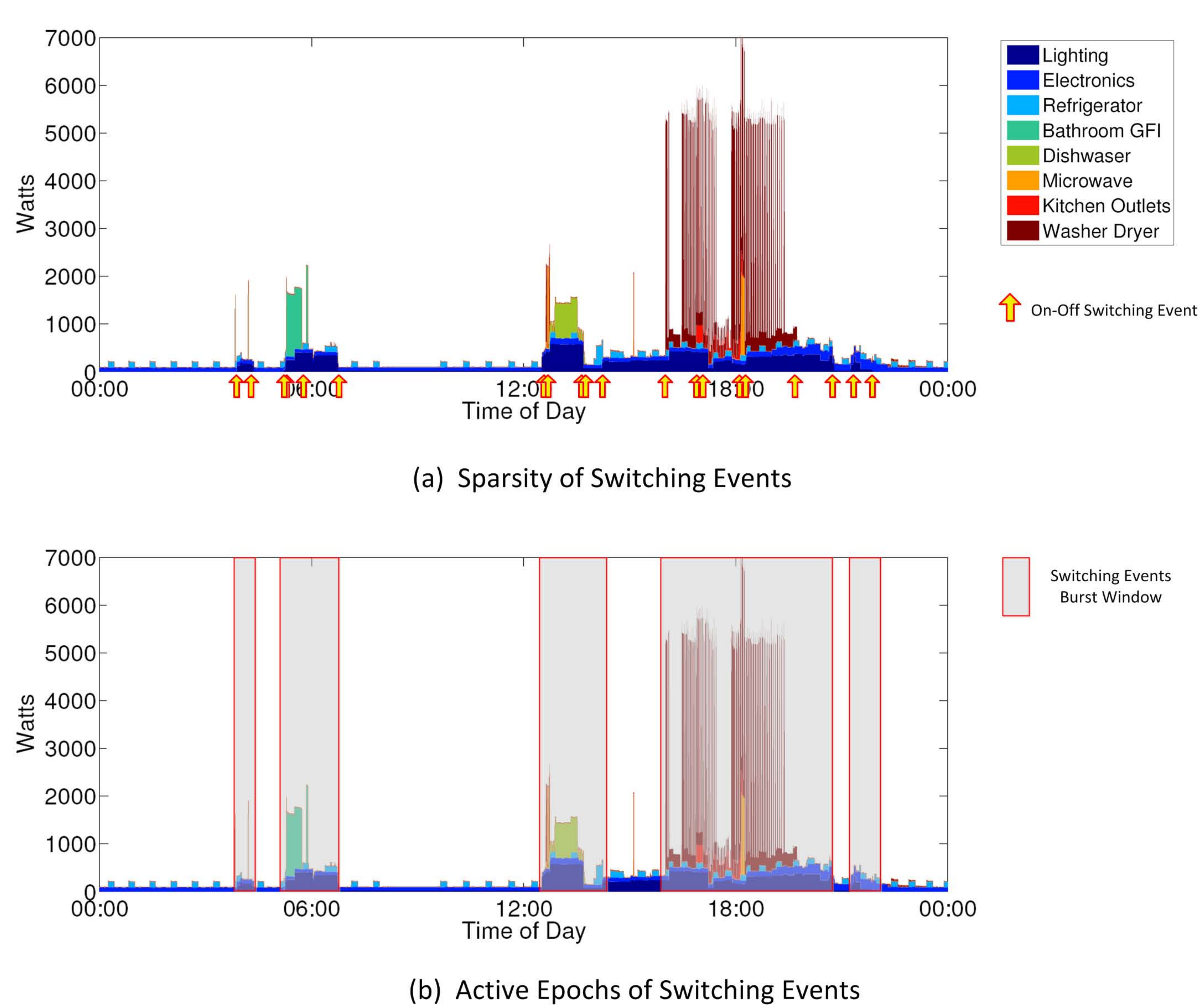}
\caption{Energy consumption and appliances' on/off switching events in a house over the course of a day~\cite{kolter2011redd}}\label{fig:illustration_sparsityNbursty}
\end{figure}

\nop{\subsection{State and Event Matrix}}

We denote the on/off states of $N$ appliances from time $t=1$ to $T$ with a \emph{state matrix}, $S$, defined as
\begin{equation}\label{StateMatrix}
S :=
  \begin{bmatrix}
    S^{(1)}_{1} & S^{(1)}_{2} & \cdots & S^{(1)}_{T}\\
    S^{(2)}_{1} & S^{(2)}_{2} & \cdots & S^{(2)}_{T}\\
		\vdots   & \vdots   & \ddots & \vdots  \\
		S^{(N)}_{1} & S^{(N)}_{2} & \cdots & S^{(N)}_{T}\\
  \end{bmatrix},
\end{equation}

\noindent in which $S^{(n)}_{t}$ represents the on/off state of the $n$-th appliance at time $t$, and $S^{(n)}_{t}\in \{0,1\}$ with $S^{(n)}_{t} = 1$ indicates the $n$-th appliance is on and $0$ otherwise. 

Then, the on/off switching events of the $N$ appliances from $t=2$ to $T$ can be indicated by an \emph{event matrix}, $\Delta{S}$, calculated as 
\begin{equation} \label{deltaS}
\Delta{S} = SD,
\end{equation}
\noindent where $D$ is a differential matrix with size of $N$-by-$(N-1)$:
\begin{equation}\label{eqt:DiffMatrix}
D :=
  \begin{bmatrix}
    -1 &    &       &   &    \\
     1 & -1 &       &   &    \\
		   &  1 &\ddots &   &    \\
		   &    &\ddots &-1 &    \\	
       &    &       & 1 & -1 \\	
			 &    &       &   &  1 \\	
  \end{bmatrix}
\end{equation}

The element of event matrix $\Delta{S}^{(n)}_{t}\in \{-1,0,1\}$, with $\Delta{S}^{(n)}_{t} = 1$ or $-1$ indicating a switching on or off event of the $n$-th appliance at time $t$, respectively, and $0$ no switching event. Since the sampling rate of current smart meters can reach $1$ to $10$ samples per second~\cite{CurrentCost2013}, we neglect the situation where an appliance has a series of switching events within a sampling interval, \textit{i.e.}, $|\Delta{S}^{(n)}_{t}| < 2$.

\textsc{Assertion $1.$} \emph{According to our real-world observations, $\Delta{S}$ is a sparse matrix.}

\section{System Model}

\subsection{Power Pattern}

We focus on the aggregated power readings of a number of appliances in a house, and arrange them from time $t=1$ to $T$ as an \emph{aggregated power vector},
\begin{equation}\label{eqt:AggLoad}
X := \left[X_{1},X_{2},\cdots,X_{T}\right]^{T}.
\end{equation}

Without loss of generality, all vectors in the paper are column vectors. Note that we slightly abuse the notation of $T$ to denote both time and the transpose of a vector/matrix. From the context, however, it is easy to figure out the difference, since when $T$ is used as the superscript of a vector/matrix, it always means the transpose of the vector/matrix in this paper.

The power pattern of an appliance indicates the energy consumption value when it is turned on or in stand-by state. In this paper, we use a simple power model which can be easily obtained from the user's guide or the specification of an appliance. We represent the power pattern of an appliance $n$ by a tuple $(I_n, P_{n},\Theta_{n})$, where $I_n$ is its stand-by power, $P_{n}$ is its rated power, and $\Theta_{n}$ is its power deviation.

Assume that a house is equipped with $N$ appliances. We define a \emph{stand-by power vector} to represent their stand-by powers as
\begin{equation}
I := \left[I_{1},I_{2},\cdots,I_{N}\right]^{T},
\end{equation}
a \emph{rated power vector} to represent their rated powers as
\begin{equation}
P := \left[P_{1},P_{2},\cdots,P_{N}\right]^{T},
\end{equation}
\noindent and a \emph{power deviation vector} to represent their power deviations as
\begin{equation}
\Theta := \left[\Theta_{1},\Theta_{2},\cdots,\Theta_{N}\right]^{T}.
\end{equation}

\nop{For a typical household, there is very low energy consumption out of active epochs, caused by stand-by or always-on appliances.}
\begin{definition}
Given a house with a certain number of appliances, we call the sum of the appliances' stand-by power, denoted by $P_{0}$, as the \textit{baseline power} of the house, \textit{i.e.,} $P_0 = \left\| I \right\|_{1}$.
\end{definition}

Note that virtually all appliances' stand-by power could be found from users' manual, technical specification or public websites such as~\cite{EnergyStar2013, Standby2014}. Theoretically, $P_{0}$ should be constant, which is the minimum power of the house at any time instant. In practice, however, there are small variations in $P_0$ due to inaccurate stand-by power specification, thus it is possible that the actual power could be below the baseline power. 

At time instant $t$, given the state vector of all appliances $S_t$, the aggregated power reading $X_{t}, (t=1, 2, \ldots, T)$, is bounded by:
\nop{by $\left[S_t^T\left(P-\Theta\right), S_t^T\left(P+\Theta\right)\right]$, \textit{i.e.}, }
\begin{equation}
\begin{aligned}
& (\textbf{1} - S_t)^T I + S_t^{T}(P-\Theta) \leq X_t, \\ 
& (\textbf{1} - S_t)^T I+ S_t^{T}(P+\Theta) \geq X_t,
\end{aligned}
\end{equation}
where $\textbf{1}$ is the all-one vector. In other words, the following constraints hold:

\begin{equation}\label{eqt:Constraints}
\begin{aligned}
&  X-S^{T}(P+\Theta) - (\mathbb{I} - S)^T I \leq \textbf{0}, \\
&  S^{T}(P-\Theta) + (\mathbb{I} - S)^T I - X \leq \textbf{0},
\end{aligned}
\end{equation}
where $\mathbb{I}$ is the $N$-by-$T$ all-one matrix.

\subsection{Sparse Switching Events Recovering}

In order to solve the energy disaggregation problem, we transform the original problem, which aims at breaking up the aggregated power readings to individual appliance at each time instant, to an alternative one, which aims at recovering the on/off states of individual appliance at each time instant. Thus, the problem is formally defined as:

\begin{itemize}
\item \textbf{Input:} Aggregated power vector $X$, power pattern $(I, P, \Theta)$. 
\item \textbf{Output:} State matrix $S$, \textit{i.e.}, the on/off states of all appliances along the timeline.
\end{itemize}

A \textbf{Sparse Switching Event Recovering (SSER)} model is established to recover the states of $N$ appliances from  time $t=1$ to $T$.
\begin{equation}\label{eqt:Minimization_1}
\begin{aligned}
&  \min{S}
& & \textbf{TV}(\Delta S)  \\
& \textit{s.t.}
& & X-S^{T}(P+\Theta) - (\mathbb{I} - S)^T I \leq \textbf{0}, \\
&&&  S^{T}(P-\Theta) + (\mathbb{I} - S)^T I - X \leq \textbf{0},
\end{aligned}
\end{equation}
\noindent where $\Delta S$ is defined by (\ref{deltaS}) and $\textbf{TV}(\cdot)$ denotes the \emph{total variation} of the event matrix calculated by 
\begin{equation}\label{eqt:XForm_1}
\textbf{TV}(\Delta S) := \sum_{n=1}^{N}\sum_{t=1}^{T}\left|\Delta{S}^{(n)}_{t}\right|.
\end{equation}

After obtaining the on/off states of each appliance along the timeline, we can estimate its power readings with its rated power at each time instant. Therefore, we can get an approximate estimation of the power consumption of each appliance. This is equivalent to solving the original energy disaggregation problem.

\subsection{On TV Minimization}

The total variation (TV) minimization is a classical approach to recovering a sparse matrix. It has been widely applied in signal restoration, image denoising, and compressive sensing~\cite{chambolle1997image,osher2003image}. To the best of our knowledge, however, it has not been explored in the context of energy disaggregation. 

Compared with the dictionary based sparse decoding~\cite{kolter2010energy}, our method does not need a training process to get the basis functions, and the recovered matrix with our method has an explicit meaning in practice. In addition, unlike other optimization methods, such as least square fitting~\cite{suzuki2008nonintrusive}, total variation minimization is a type of least absolute deviations fitting, which has been proved to be more robust for various applications~\cite{boyd2011distributed}. \nop{Considering the information on rated power might not be very accurate, the TV model matches our problem well, as shown in our later robust test results. and is expected to solve energy disaggregation precisely.}

\section{Parallel Local Optimization\\ Over Active Epochs} \label{sec:PLOA}

In this section, we first analyze the hardness of solving SSER. To solve the problem efficiently, we then propose a Parallel Local Optimization Algorithm (PLOA) by splitting the whole timeline into multiple active epochs. 

\subsection{Hardness of SSER}

There were significant research efforts to solve the total variation minimization problem~\cite{chambolle2004algorithm,figueiredo2006total}. Nevertheless, the form of total variation in our case is a discrete version and involves integer variables. Since $S^{(n)}_{t}\in\{0,1\}$, our problem belongs to binary programming, which is much harder to solve than the one with real variables. 

With $T$ aggregated power readings generated by $N$ appliances, to obtain the optimal solution via a brute-force method, it can be shown that the computational complexity is $O\left(2^{N\cdot{T}}\right)$ (see Appendix~\ref{Appendix_B}), which is exponential. Furthermore, we can show that solving SSER is NP-hard (see Appendix~\ref{Appendix_A}). Since $T$ is very big in practice, it seems not possible to find an efficient algorithm that outputs the optimal solution. Nevertheless, the active epochs of on/off events suggest that we can perform optimization in a smaller, local time window.\nop{ which inspires us to design an algorithm to solve the problem perfectly.}

\subsection{Detection of Active Epochs}\label{burstwindow}

\begin{definition} An \textit{active epoch} of a house is defined as a time interval from the time when the aggregated power of the house jumps above the baseline power until the time when the aggregated power drops below the baseline power. 
\end{definition}

\begin{figure}[!ht]
\centering
\includegraphics[width=3.2in,height=1.0in]{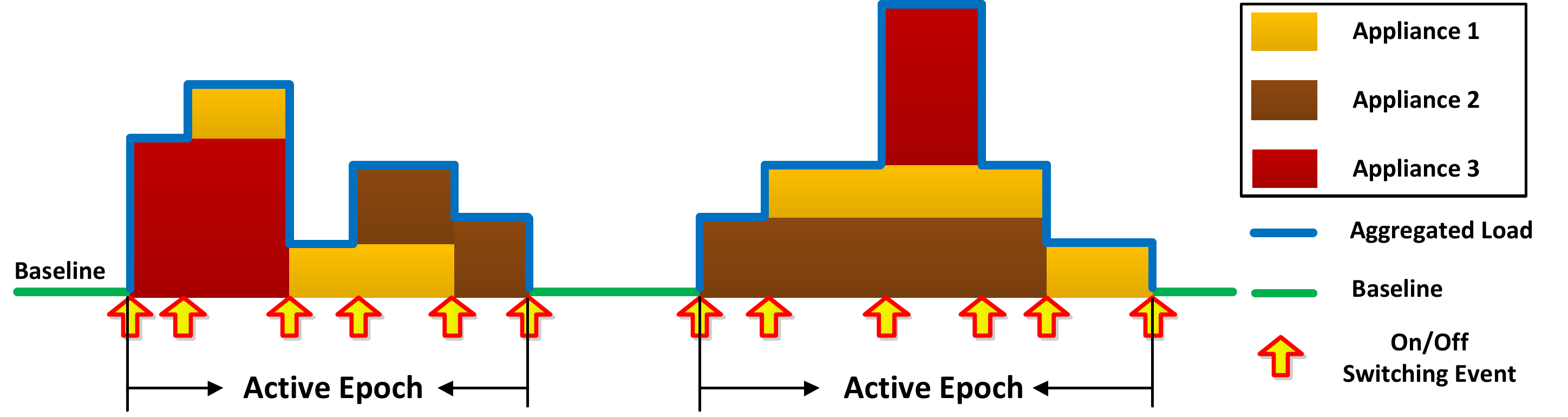}
\caption{A sketch map to illustrate the concepts of active epoch and baseline power using three appliances}\label{fig:illustration_baselineNburstwindow}
\end{figure}

Fig.~\ref{fig:illustration_baselineNburstwindow} is a sketch map of switching activities and power readings of three appliances with constant power, in which the concepts of baseline power and active epoch are illustrated. \nop{Therefore, under our definition and assumption, we can detect active epochs along the timeline by the following strategy.}

Algorithm~\ref{alg_windowDet} shows the pseudo code of detecting active epochs. 

\begin{algorithm}[!ht]
\caption{Active Epoch Detection}\label{alg_windowDet}
\begin{algorithmic}[1]
\Require Aggregated power vector $X$, baseline power $P_{0}$.
\Ensure Set of Active epochs, $W$.
\State $t = 1,k = 0$
\While {$t \leq T$}
	\State $start = end = t$
	\While {$X_{end} > P_{0}$ and $end < T$}
		\State $end = end + 1$
	\EndWhile
	\If {$end > start$}
		\State $k = k + 1$
		\State $W_{k} = [start, end]$
	\EndIf
	\State $t = end + 1$
\EndWhile
\State \textbf{return} $W = \{W_{1},W_{2},\cdots,W_{k}\}$
\end{algorithmic}
\end{algorithm}

\subsection{Parallel Local Optimization Algorithm}

Without loss of generality, we take aggregated load data of $N$ appliances from time $t = 1$ to $T$ as an example to show the major steps of PLOA.

\textbf{Step 1:} Detect all active epochs along the timeline with Algorithm~\ref{alg_windowDet}. Denote the set of active epochs as $W = \{W_{1},W_{2},\cdots,W_{k}\}$. 

\textbf{Step 2:} In the active epoch starting at $t$ with the length of $\ell$, solve the following optimization problem to obtain $S_{t:t+\ell}$.
\begin{equation}\label{eqt:Minimization_2}
\begin{aligned}
& \min{S_{t:t+\ell}} \textit{   } \textbf{TV}(S_{t:t+\ell}D_{t:t+\ell})  \\
& \textit{s.t.} \\
& X_{t:t+\ell}-(S_{t:t+\ell})^{T}(P+\Theta) - (\mathbb{I}_{t:t+\ell} - S_{t:t+\ell})^T I \leq \textbf{0}, \\
& (S_{t:t+\ell})^{T}(P-\Theta) + (\mathbb{I}_{t:t+\ell} - S_{t:t+\ell})^T I - X_{t:t+\ell} \leq \textbf{0},
\end{aligned}
\end{equation}
where $S_{t:t+\ell}$ is a $N$-by-$\ell$ submatrix of $S$, $D_{t:t+\ell}$ is a $\ell$-by-$(N-1)$ submatrix of $D$, $\mathbb{I}_{t:t+\ell}$ is a $N$-by-$\ell$ submatrix of $\mathbb{I}$, and $X_{t:t+\ell}$ is a vector containing the aggregated power readings of all appliances from time $t$ to time $t+\ell$. 

\textbf{Step 3:} Perform Step $2$ on the $k$ active epochs to obtain a group of $k$ solutions in parallel. Since outside of active epochs, appliances are considered as stand-by, a complete state matrix $S_{1:T}$ can thus be built. \nop{in which the $t$-th ($1\leq t \leq T$) column is: \nop{(supposing the idle state is $S_{0}$)}

\begin{equation}\label{eqt:SLOAsolution}
\hat{S}_{t} = \left\{ 
	\begin{array}{ll} 
		S_{t} \text{ , } t \in W,\\
		S_{0} \text{ , } t \notin W.
	\end{array} \right.
\end{equation}
}

We can show that the computational complexity to solve (\ref{eqt:Minimization_2}) is $O(2^{N\cdot{\ell}})$ (see Appendix~\ref{Appendix_B}). Since $\ell \ll T$ as shown in Fig~\ref{fig:illustration_sparsityNbursty}, the problem can be solved efficiently, using tools such as \emph{CVX} $2.0$ with a \emph{Gurobi} engine~\cite{cvx2013}. \nop{ Algorithm~\ref{alg_seqentialOpt} shows the pseudo code of PLOA, in which \textit{parfor}\footnote{With support of \emph{Parallel Computing Toolbox}, \emph{parfor} loop can be applied in \emph{MATLAB} and bring in multiple \textit{workers} to execute codes parallelly.} is a notation of parallel computing instead of \emph{for}. 


\begin{algorithm}[!ht]
\caption{Parallel Local Optimization}\label{alg_seqentialOpt}
\algblock{ParFor}{EndParFor}
\algnewcommand\algorithmicparfor{\textbf{parfor}}
\algnewcommand\algorithmicpardo{\textbf{do}}
\algnewcommand\algorithmicendparfor{\textbf{end\ parfor}}
\algrenewtext{ParFor}[1]{\algorithmicparfor\ #1\ \algorithmicpardo}
\algrenewtext{EndParFor}{\algorithmicendparfor}

\begin{algorithmic}[1]
\Require Aggregated power vector $X$, rated power vector $P$ and deviation vector $\Theta$, active epoch set $W$.
\Ensure State matrix $\hat{S}$.

\ParFor {$n = 1:k$}
	\State $t = W_{k}[1]$
	\State $\ell = W_{k}[2] - W_{k}[1]$
	\State Solve (\ref{eqt:Minimization_2}) to obtain $S_{1:1+\ell}$
	\State $\hat{S}_{t:t+\ell}=S_{1:1+\ell}$
\EndParFor
\State \textbf{return} $\hat{S} = \{\hat{S}_{t} \text{ fits (\ref{eqt:SLOAsolution})}\}$.
\end{algorithmic}
\end{algorithm}
}
\newtheorem{theorem}{Theorem}
\begin{theorem}
Assume that the global optimal solution to SSER in (\ref{eqt:Minimization_1}) is $S^{*}$, and the solution obtained from POLA is $\hat{S}$, if both solutions are unique, then $\hat{S} = S^{*}$.
\end{theorem}

\begin{proof}

For an arbitrary active epoch starting at $t$ with the length of $\ell$, assume that $\hat{S}_{t:t+\ell}$ is the unique local optimal solution obtained via (\ref{eqt:Minimization_2}). Assume that the global optimal solution to SSER in (\ref{eqt:Minimization_1}) is $S^{*}$. Assume that the sub-matrix constructed by the $t$-th to $(t+\ell)$-th columns of $S^{*}$  is $S^{*}_{t:t+\ell}$. We prove the theorem by contradiction. 

Assume that 
\begin{equation}\label{eqt:assumption}
\hat{S}_{t:t+\ell} \neq S^{*}_{t:t+\ell}.
\end{equation}

Then, the following inequality must hold
\begin{equation}
S^{*}_{t:t+\ell}D_{t:t+\ell} \geq \hat{S}_{t:t+\ell}D_{t:t+\ell}.
\end{equation}

Therefore, there must exist another global solution $S^{**}$, in which the $j$-th column is
\begin{equation}
S^{**}_{j} = \left\{ 
	\begin{array}{ll} 
		\hat{S}_{j} \text{ , } j\in[t,t+\ell], \\
		  S^{*}_{j} \text{ , } j\notin[t,t+\ell],
	\end{array} \right.
\end{equation}

\noindent such that
\begin{equation}\label{eqt:contradictory}
S^{**}D \leq S^{*}D.
\end{equation}

Obviously, (\ref{eqt:contradictory}) is contradictory to the assumption that $S^{*}$ is the \textit{uniquely} global optimal solution to SSER. Therefore, the assumption (\ref{eqt:assumption}) is not true. As a result, we have
\begin{equation}
\hat{S}_{t:t+\ell} = S^{*}_{t:t+\ell}.
\end{equation}

\nop{Similarly, we can prove that any local optimal solution obtained via PLOA is equal to the corresponding sub-matrix of the global optimal solution.} Outside the active epochs, PLOA treats all appliances as stand-by, the TV value is $0$ in $\hat{S}$.  Since the TV value cannot be negative,  the TV value obtained with PLOA is the minimum and must be the same as that obtained with the global optional solution. \nop{in the optimal solution $S^{*}$. According to (\ref{eqt:SLOAsolution}), $S_{0}$ is also included in $\hat{S}$. }

Overall, if the global optimal solution is unique, for any time instant $t$, no matter whether $t$ is in an active epoch or outside active epochs, the state vector $\hat{S}_{t} \in \hat{S}$ must be equal to the state vector $S^{*}_{t} \in S^{*}$, which means $\hat{S} = S^{*}$.
\end{proof}

\subsection{Algorithm Analysis}

Given $T$ aggregated power readings generated by $N$ appliances that can be broken into $k$ active epochs with maximum size $w$, the computational complexity of the original SSER problem~(\ref{eqt:Minimization_1}) is $O(2^{N\cdot{T}})$. With PLOA, solving the local optimization problem~(\ref{eqt:Minimization_2}) $k$ times results in the time complexity upper bounded by $O(k\cdot 2^{N\cdot{w}})$ (see \emph{Appendix B}). Considering that the number of appliances $N$ is constant and $w \ll T$, PLOA significantly cuts down the computational complexity.

Obviously, the larger the value of $w$, the higher the computational complexity. Fortunately, in practice, each active epoch in a house is usually not long. \nop{Therefore, in practice, PLOA can solve the NP-hard problem efficiently.} As we will show in later experiments, PLOA can indeed provide satisfied solutions.

\section{Data Collection \nop{and Pre-processing}}

\begin{figure*}[!ht]
\centering
\includegraphics[width=6.5in,height=1.8in]{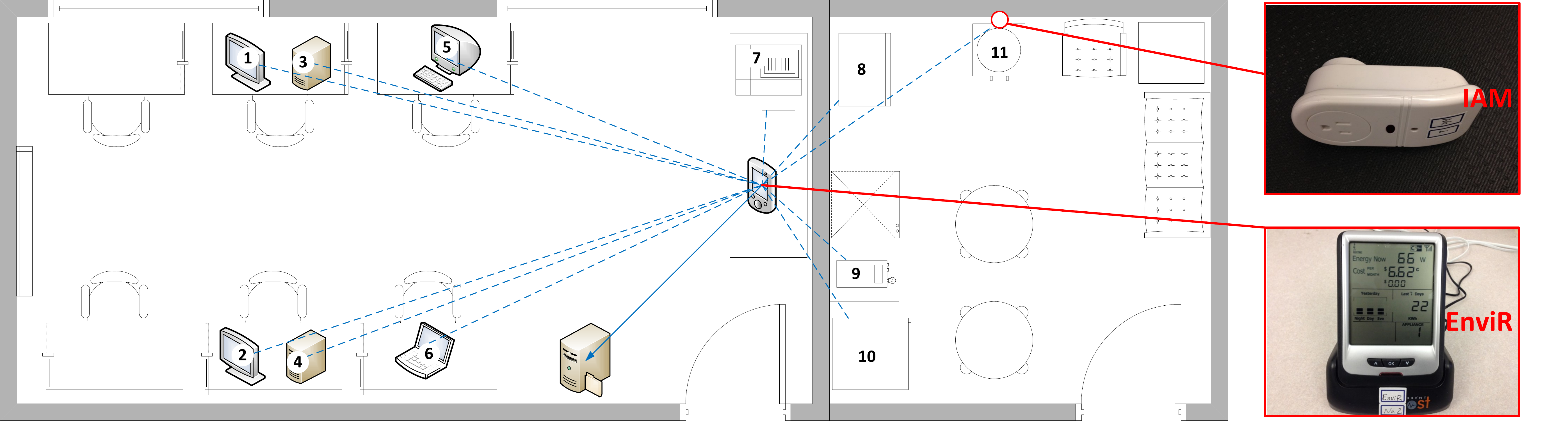}
\caption{Energy monitoring platform, monitored appliances and measuring devices}\label{fig_appDeployment}
\end{figure*}

\subsection{Energy Monitoring Platform}\label{subsec:dataCollection}

We evaluated our method with real-world trace data from our energy monitoring platform.\nop{\footnote{We are unable to make use of REDD datasets, because they 1) lack phase information of sources, 2) are rather circuit-oriented than appliance-oriented, and 3) contain duplicate labels.}.} We monitored the appliances' energy consumption in a typical laboratory and a lounge room in the fifth floor of Engineering/Computer Science building at the University of Victoria (UVic).

Using an off-the-shelf solution developed by Current Cost (\textit{http://www.currentcost.com}), we recorded the real-time power of laptops, desktops and some household appliances. Each appliance's real power was measured every $6$ seconds by the device called Individual Appliance Monitor (IAM), and the measurement results were transmitted via wireless to a display server (EnviR), which can display and temporarily store the collected data. Then, the data in EnviR were sent to our data server.\nop{ every $10$ seconds.} The platform, the monitored appliances, and the measuring devices are shown in Fig.~\ref{fig_appDeployment}.

We collected the data for three months, and one-week data were used for performance evaluation in Section~\ref{sec:evaluation}. 

\subsection{Power Splitting}\label{subsec:powerSplitting}

If the power range of appliance $A_{1}$ largely overlaps with the power range of appliance $A_2$, given a power value in the overlapping range, it would be hard to decide which appliance is on. To alleviate this problem, we should reduce the overlapping power range of two different appliances. This is achieved by a power splitting method as follows. 

We can find that some simple appliances like a bulb or a stove, once turned on, usually have stable power readings with small fluctuation. In contrast, complex appliances such as refrigerator usually have multiple working modes, and the power readings at each mode tend to be stable. Therefore, we can split the power consumption of complex appliances into multiple modes, each of which is regarded as a \textit{virtual appliance}. With such power splitting, the power overlaps among virtual appliances can be narrowed down significantly.

With readily available appliance information from user's manual or public websites such as~\cite{EnergyStar2013}\nop{\textit{e.g.}, the rated power of a kind of microwave illustrated in Fig.~\ref{fig:userManual}\nop{\footnote{Available at \url{http://www.panasonic.co.uk/html/en_GB/Products/Kitchen+Appliances/Microwave+Combination/NN-CT579SBPQ/Specification/4923503/}}}, or public websites with expertise~\cite{EnergyStar2013}}, we can easily split the power range of an appliance. The splitting result in our test scenario is given in Table.~\ref{tab:powersplitting}, where the values of power deviations are estimated from the collected power data of each appliance. One may be concerned that the estimation of power deviation in practice is inaccurate. With experimental study, however, our method is resilient to inaccurate power deviation estimations as shown in Section~\ref{sec:evaluation}.

\nop{
\begin{figure}[!ht]
\centering
\includegraphics[width=3.2in,height=1.5in]{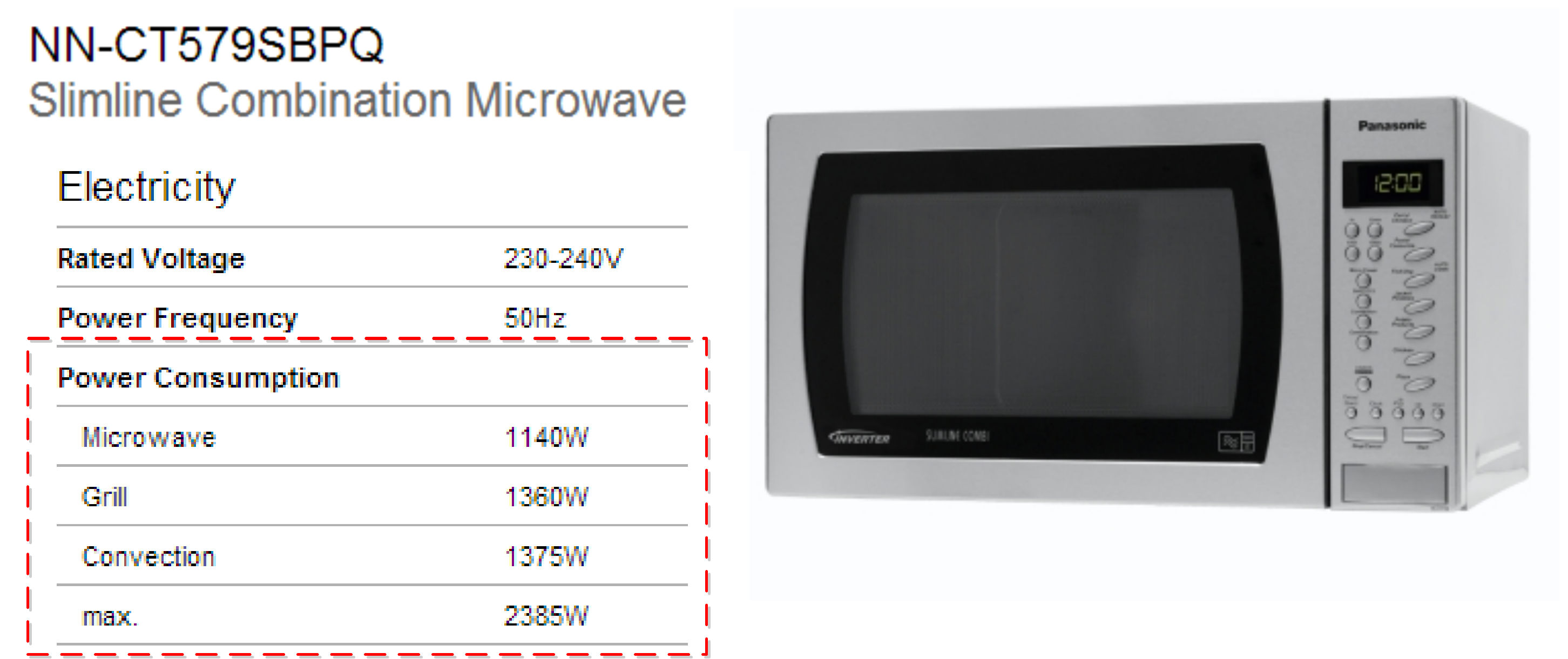}
\caption{Power consumption of a kind of microwave specified in the user manual}\label{fig:userManual}
\end{figure}
}

\begin{table}[!ht]
	\caption{Results of power splitting for each appliance}\label{tab:powersplitting}
	\centering
	\begin{small}
		\begin{tabular}{|c|m{1.5cm}|m{0.6cm}|m{1.2cm}|m{1.3cm}|m{1.2cm}|}
			\hline
			ID & Appliance & Mode & Rated Power (\emph{Watts}) & Power Deviation (\emph{Watts}) & Stand-by Power (\emph{Watts})\\
			\hline
			\hline
			 1 & LCD-Dell & 1 & 25 & 5 & 0 \\
			\hline
			 2 & LCD-LG & 1 & 22 & 5 & 0 \\
			\hline
			 3 & Desktop & 1 & 40 & 15 & 3 \\
			   &         & 2 & 50 & 20 & \\
			\hline
			 4 & Server & 1 & 130 & 20 & 10\\
			\hline
			 5 & iMac & 1 & 35 & 5 & 3\\
			   &      & 2 & 50 & 10 & \\
			\hline
			 6 & Laptop & 1 & 15 & 5 & 1 \\
			   &        & 2 & 30 & 10 & \\
				 &        & 3 & 70 & 10 & \\
			\hline
			 7 & Printer & 1 & 400 & 50 & 2\\
			   &         & 2 & 700 & 80 & \\
				 &         & 3 & 900 & 100 & \\
			\hline
			 8 & Microwave & 1 & 1000 & 100 & 2\\
			   &           & 2 & 1200 & 100 & \\
				 &           & 3 & 1700 & 100 & \\
			\hline
			 9 & Coffee Maker & 1 &  700 & 100 & 2\\
			   &              & 2 &  900 & 100 & \\
				 &              & 3 & 1100 & 100 & \\
			\hline
			 10 & Refrigerator & 1 & 115 & 15 & 5\\
					&              & 2 & 350 & 10 & \\
			\hline
			 11 & Water Cooler & 1 & 65 & 5 & 3\\
					&              & 2 & 380 & 10 & \\
					&              & 3 & 450 & 10 & \\
			\hline
		\end{tabular}
	\end{small}
\end{table}

After power splitting, the original state vector of an appliance with $k$ modes is extended to a state matrix with $k$ rows, each representing the state vector of a virtual appliance. In consequence, the sum of states of multiple virtual appliances split from the same real appliance may be larger than one.  To avoid this problem, we add an extra integer constraint in our model if the $n$-th appliance has $k$ different modes:
\begin{equation}\label{eqt:added_constraint}
\sum_{i=1}^{k}S^{(\Gamma^{n}_{i})}_{t}\leq 1, 
\end{equation}
\noindent where $\Gamma^{n}$ is the set of modes of the $n$-th appliance and $\Gamma^{n}_{i}$ is the corresponding row number in state matrix for mode $i$. Thus, $S^{(\Gamma^{n}_{i})}_{t} = 1$ indicates that at time $t$ the $n$-th appliance is turned on and working in mode $i$; otherwise, $S^{(\Gamma^{n}_{i})}_{t} = 0$. The constraint in (\ref{eqt:added_constraint}) means that at any time instant, the appliance can only work in one mode. 

\section{Experimental Evaluation}\label{sec:evaluation}
Using real-world trace data collected from our energy monitoring platform, we evaluate our method 
by 1) comparing its performance with others', and 2) testing its robustness with inaccurate parameters.

\subsection{Comparison}
For comparison purpose, we also implement and test another two methods: the Least Square Estimation based integer programming method~\cite{suzuki2008nonintrusive} and the iterative Hidden Markov Model~\cite{parson2012non}. The former is a signature based approach, while the latter is a state transition based approach.

\subsubsection{Least Square Estimation Based Integer\\Programming}
The Least Squire Estimation (LSE) based integer programming method was adopted in~\cite{suzuki2008nonintrusive}. The current waveform of each appliance was extracted and stored beforehand, and treated as its signature for energy disaggregation. Since it needs extra devices to obtain the appliances' current waveform, for a fair comparison, we use the rated power listed in Table~\ref{tab:powersplitting} instead of current waveform as the appliances' signatures and implement the LSE based algorithm as in~\cite{suzuki2008nonintrusive}. To be specific, with the same notations mentioned in~(\ref{eqt:Minimization_1}) and~(\ref{eqt:added_constraint}), the LSE-based method in our scenario is formally defined as:
\begin{equation}\label{eqt:Minimization_3}
\begin{aligned}
&  \min{S}
& & \left\|X-S^{T}P\right\|_{2}\\
& \textit{s.t.}
& & S^{(\Gamma^{n}_{i})}_{t} \in\{0, 1\}, \sum_{i=1}^{k}S^{(\Gamma^{n}_{i})}_{t}\leq 1,\\
&&& 0 \leq n \leq N, 0 \leq t \leq T.
\end{aligned}
\end{equation}

By solving the above optimization problem, we can get the states (modes) of appliances at each time instant, and estimate the energy consumption of each appliance using its rated power.

\subsubsection{Iterative Hidden Markov Model}
As a state transition based method, the iterative Hidden Markov Model (HMM) was proposed and tested for energy disaggregation in~\cite{parson2012non}. \nop{Compared with the traditional HMM, the iterative HMM used a variant of the difference HMM, and the step changes in aggregated data and the states (with multiple modes) of each appliance were modeled as the observed and hidden variables, respectively. } We implement this method in three phases: the modelling phase, the training phase, and the inference phase.
\begin{itemize}
\item In the modeling phase, each appliance is modelled as a prior difference HMM, which is defined by
\begin{equation}
\lambda := \{A, B, \pi\},
\end{equation}
\noindent where $A$ is the prior state transition probability distribution, $B$ is the emission probability distribution, and $\pi$ is the starting state distribution of the appliance. In particular, 1) $A$ is initialized with the transition probabilities proportional to the time spent in each state, and 2) for any change between states (or modes) $\Gamma^{n}_i$ and $\Gamma^{n}_j$ of the $n$-th appliance, its corresponding emission probability in $B$ is defined by a Gaussian distributed power consumption $\mathcal{N}(P_{\Gamma^{n}_i}-P_{\Gamma^{n}_j},\Theta^2_{\Gamma^{n}_i}+\Theta^2_{\Gamma^{n}_j})$, where $P_{\Gamma^{n}_i}$ and $\Theta_{\Gamma^{n}_i}$ denote the rated power and power deviation of the $i$-th appliance under state (or mode) $i$, respectively.
\item In the training phase, the prior appliance model $\lambda$ is tuned by running the expectation maximization (EM) algorithm over the collected load data~\cite{parson2012non}.  The EM algorithm is initialized with the prior state transition matrix $A$ and individual appliances' rated power in Table~\ref{tab:powersplitting}. It terminates when a local optima in the log likelihood function is found or the maximum number of iterations ($100$ in our implementation) is reached.
\item In the inference phase, the extended Viterbi algorithm~\cite{parson2012non} was applied to infer each appliance's states (or modes), considering the constraints of aggregated power and power changes at each time instant.
\end{itemize}

By running the above three phases iteratively on each appliance, we can get the estimated states as well as the power consumption of each appliance at each time instant.

\begin{table*}[!ht]
	\caption{Accuracy and overhead of energy disaggregation, using Sparse Switching Event Recovering (SSER), Least Square Estimation (LSE) based integer programming and iterative Hidden Markov Model (HMM)}
	\centering
		\begin{tabular}{l|c c|c c c}
			\hline
			\multirow{2}*{\backslashbox{Methods}{Metrics}} & \multicolumn{2}{C{5cm}|}{\textbf{Accuracy}} & \multicolumn{3}{c}{\textbf{Overhead}}\\
			              & \emph{EDA} & \emph{SPA} & \emph{Training Size} & \emph{R.T.}(\emph{second}) & \emph{RAM}(\emph{MB}) \\
			\hline
			\textbf{SSER} & $61.12\%$ & $69.62\%$ & -- & $865.4$ & $596.8$ \\

			\textbf{LSE}  & $33.40\%$ & $45.67\%$ & -- & $619.3$ & $581.9$ \\

			\textbf{HMM} (average) & $55.27\%$ & $67.47\%$ & 2116 & $3721.3$ & $558.6$ \\
			
			\textbf{HMM} (best) & $67.26\%$ & $71.29\%$ & 600 & $1299.7$ & $557.9$ \\
			
			\textbf{HMM} (worst) & $41.09\%$ & $61.27\%$ & 3200 & $7089.6$ & $561.4$ \\
			\hline
		\end{tabular}
		\label{tab:resultComparison}
\end{table*}

\begin{figure*}[!ht]
\centering
\includegraphics[width=6.2in,height=2in]{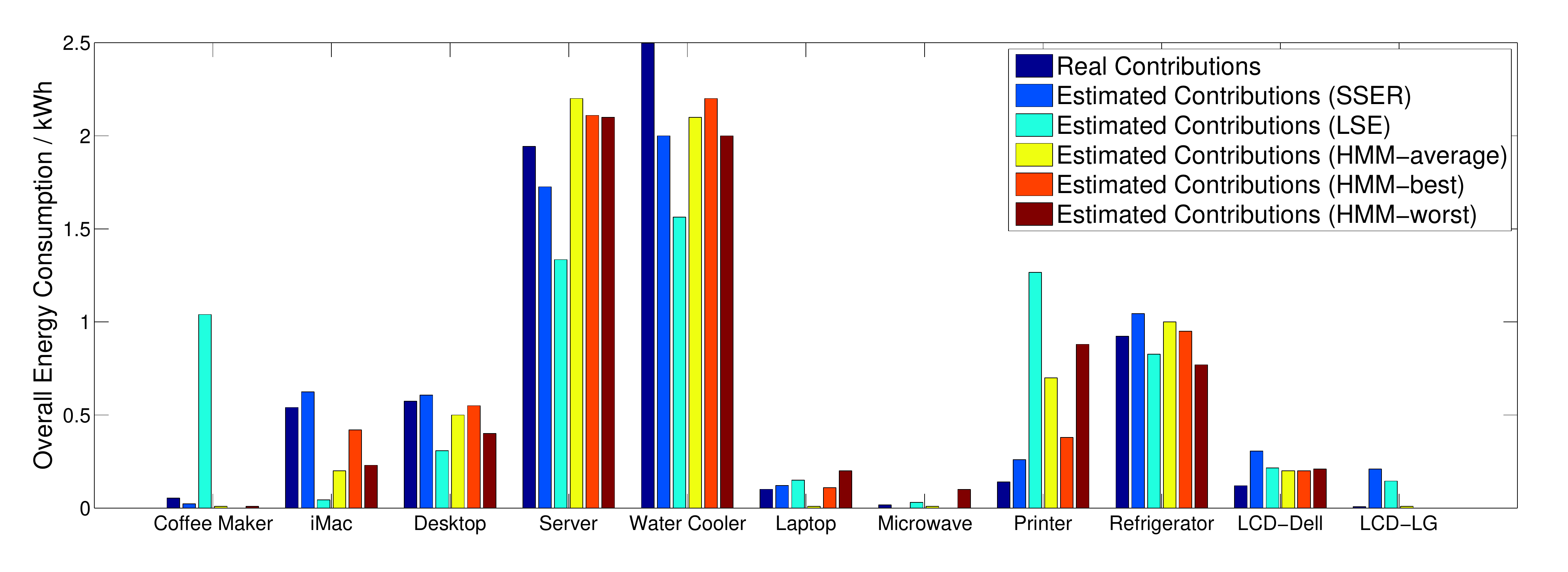}
\caption{Actual and estimated energy contributions of each appliance to the total energy consumption for one-week time period.}\label{fig:overall_consumption}
\vspace{-0.1in}
\end{figure*}

\subsection{Performance Evaluation}

To evaluate the error of energy disaggregation, the \emph{Disaggregation Error} is usually used~\cite{kolter2011redd,parson2012non,figueiredo2013regularization}. Furthermore, since we know appliances' states (\textit{i.e.}, the ground truth) in our dataset,  we also evaluate the accuracy of recovered appliances' states via \emph{Hamming Loss}~\cite{elisseeff2001kernel}. Accordingly,  we use $1-Disaggregation Error$ and $1-Hamming Loss$ to get the \emph{accuracy} of energy disaggregation and the accuracy of state estimation, respectively. In addition, we also look into the overhead of each method. The evaluation metrics are defined as follow.

\begin{itemize}
\item \emph{Energy Disaggregation Accuracy (EDA)}: It indicates the accuracy of assigning correct power values to corresponding appliances. 
\begin{equation}
EDA := 1-\frac{\sum_{n=1}^{N}\left\|X^{(n)}-\hat{S}^{(n)}P_{n}\right\|_{1}}{\left\|X\right\|_{1}},
\end{equation}
\noindent where $X^{(n)}$, $\hat{S}^{(n)}$ and $P_{n}$ represent the true energy consumption vector, the estimated state vector, and the rated power of the $n$-th appliance, respectively, and $X$ is the aggregated power vector.
\item \emph{State Prediction Accuracy (SPA)}: It indicates the accuracy of estimating the states of appliances.
\begin{equation}
SPA := 1- \frac{\sum_{n=1}^{N}\left\|S^{(n)}-\hat{S}^{(n)}\right\|_{1}}{N\cdot{T}},
\end{equation}
\noindent where $S^{(n)}, \hat{S}^{(n)}$ represent the true state vector and the estimated state vector of the $n$-th appliance, respectively, and $N,T$ represent the number of appliances and the number of samples, respectively.
\item Running time (\emph{R.T.}) and memory usage (\emph{RAM})\footnote{We implemented the three methods with \emph{Matlab} $8.0$ and run them with $32$-bit Windows OS with $3.4 GHz$ CPU and $4 GB$ RAM.}: They indicate the overhead on running time and memory space, respectively.
\end{itemize}

Since the performance of the iterative HMM method depends on model training, we run this method multiple times over different sizes (w.r.t. number of samples) of training datasets (denoted as \emph{training size}). To be specific, we changed the training size from $200$, increased by $200$ each time, up to $4000$. The average performance is calculated over all the runs, and the best and the worst performance is the best and the worst outcomes among all the runs, respectively.   

The performance of the three methods on energy disaggregation are summarized in Table~\ref{tab:resultComparison}. In addition, as illustrated in Fig.~\ref{fig:overall_consumption}, we also look into the \emph{overall energy disaggregation accuracy} of the three methods, which indicates the energy contribution of each appliance to the total energy consumption in the whole time period.

From the results, we can draw the following conclusions.
\begin{itemize}
\item In term of accuracy, our SSER method performs much better than the LSE based method and slightly better than the iterative HMM method in average. \nop{for either the individual accuracy (\emph{EDA} and \emph{SPA}) for each appliance at each time instant or the overall accuracy in the whole time period.}
\item In term of overhead, our SSER method and the LSE method are at a comparative level for running time and system memory usage. While the memory usage of the iterative HMM method is similar to that of the other two methods, its running time is much longer. \nop{while with extra training, the HMM method consumes much more time than the other two, even though it takes slightly less system memory usage.}
\item The performance of the iterative HMM method is subject to the training process and may have a large variation in accuracy and running time.  
\end{itemize}

\subsection{Robustness Test}

Regarding the iterative HMM method, as shown in Table~\ref{tab:resultComparison}, we have found that 1) the gap between the best and the worst outcomes is significant, and 2) there is no direct relationship between the size of training dataset and the estimation accuracy.  These indicate that the iterative HMM method is sensitive to  parameter estimation in the training phase.  Consequently, when using this method in practice, it is not easy  to estimate appropriate model parameters that can guarantee the performance. This problem is severe especially when the ground truth of energy disaggregation is unknown and thus it is hard to judge whether or not a trained model is good enough.

\nop{
\begin{figure}[!ht]
\centering
\includegraphics[width=3.4in,height=1.4in]{hmmResult.pdf}
\caption{The accuracy metrics resulted from the HMM method over different size of training datasets.}\label{fig:hmmResult}
\end{figure}
}

The above findings motivate us to perform robustness test on our method. In practice, we have shown in Section~\ref{subsec:powerSplitting} that the rated power of an appliance can be easily learned. However, we may not precisely estimate the power deviation of an appliance working under a certain mode. Due to this consideration, we test the performance of our method, assuming that the power deviations of appliances are not accurate.

For this test, we replace $\Theta$ with $\rho\cdot{\Theta}$, so that we can narrow down or widen up the estimated power deviations by regulating $\rho$. The value of $\rho$ is changed from $0.8$ to $1.2$, causing a parametric error of power deviation up to $20\%$.

\begin{table}[!ht]
	\caption{Accuracy of Energy Disaggregation using SSER, with inaccurate estimation on power deviation}
	\centering
		\begin{tabular}{c| m{2cm} m{2cm}}
			\hline
			\backslashbox{$\rho$}{Metrics} & \emph{EDA} & \emph{SPA} \\
			\hline
						$0.8$ & $55.28\%$ & $70.37\%$ \\
						$0.9$ & $60.33\%$ & $70.27\%$ \\	
						$1.0$ & $61.12\%$ & $69.62\%$ \\	
						$1.1$ & $56.94\%$ & $71.15\%$ \\	
						$1.2$ & $59.59\%$ & $72.24\%$ \\
			\hline
		\end{tabular}
		\label{tab:faultTolerance}
\end{table}

\nop{
\begin{table*}[!ht]
	\caption{Accuracy of Energy Disaggregation using SSER, with inaccurate estimated parameter of power deviation}
	\centering
		\begin{tabular}{|c|c| m{1.8cm} m{1.8cm} m{1.8cm} m{1.8cm} m{1.8cm} |}
			\hline
			\multicolumn{2}{|r|}{\backslashbox{Metrics}{$\rho$}} & $0.8$ & $0.9$ & $1.0$ & $1.1$ & $1.2$  \\
			\hline
			\textbf{Accuracy} & \emph{EDA} & $55.37\%$ & $60.12\%$ & $60.49\%$ & $56.12\%$ & $57.41\%$  \\
			
			(before data cleansing) & \emph{SPA} & $65.71\%$ & $68.01\%$ & $69.03\%$ & $70.72\%$ & $69.10\%$  \\
			\hline
			\textbf{Accuracy} & \emph{EDA} & $55.28\%$ & $60.33\%$ & $61.12\%$ & $56.94\%$ & $59.59\%$  \\
			
			(after data cleansing)  & \emph{SPA} & $70.37\%$ & $70.27\%$ & $69.62\%$ & $71.15\%$ & $72.24\%$  \\
			\hline
		\end{tabular}
		\label{tab:faultTolerance}
\end{table*}
}

Part of the outcomes are shown in Table~\ref{tab:faultTolerance}. We can see that the accuracy does not change too much when the parameter error varies, indicating that our method is robust to  parameter estimation.

\section{Further Discussion: Limitation of Our Method}

While our method is effective and simple, it has the following limitations. 

First, we assumed that the types and the number of appliances in a house are known \textit{a prior}. While this assumption is reasonable if our method is used by the residents, it may not hold for the utility side because the household may be concerned of privacy and thus unwilling to provide the above information. In this case, our method may not work well; at least it will need the help of other sophisticated methods to detect the appliances used in the house first.   

Second, our method cannot automatically adjust if new appliances are added to or existing appliances are removed from the house. This problem is similar to the first one. We need to assume that the household is collaborative and should inform the system whenever there is a change on the major appliances. Otherwise, our method would not work well. 

Finally, we assumed that the set of appliances is stable during the time in consideration. This may not be always true, since a house may include some small \textit{ad hoc} devices, \textit{e.g.}, mobile phones, which are charged and then unplugged from the power sockets. Fortunately, these ad hoc devices may not have perceptible impact on the energy disaggregation results in practice, because their energy consumption is usually not high compared to other stable appliances such as refrigerator and stove.  

\section{Conclusions}

Most existing approaches for energy disaggregation either require complex appliances' signatures or use machine learning techniques to train a ``good" model. Their effectiveness has been questioned by the lack of commonly-accepted signatures or by the fluctuation in estimation results due to the difficult model training process. It has been challenging to develop a simple and broadly applicable method in this important application domain.  

In this paper, we proposed a simple, universal model for energy disaggregation.  We only make use of readily available information of appliances, \textit{e.g.}, those from users' manual, technical specification, or some public websites. We built a sparse switching event recovering model, based on the sparsity of appliances' switching events. Furthermore, we used the active epochs of switching events to develop a parallel local optimization algorithm to solve our model efficiently. In addition to analyzing the complexity and correctness of our algorithm, we tested our method with the real-world trace data from an energy monitoring platform we deployed, which records the power readings from a group of household appliances. The test results demonstrated that our method can achieve better performance than the state-of-the-art solutions, including the Least Square Estimation (LSE) method and the machine learning method using iterative Hidden Markov Model (HMM). 

\bibliographystyle{abbrv}
\bibliography{reference}  


\appendix
\section{Proof of NP-hardness of SSER}\label{Appendix_A}

\subsection{Preparation}
First, we introduce a tree structure $\mathcal{T}(M, T)$, which is a complete $M$-ary tree with height of $T$, \textit{i.e.}, every internal node has exactly $M$ children and all leaves have the same depth of $T$. By default, the height of the root is $0$. Furthermore, each edge $(i,j)$ of $\mathcal{T}$ has a non-negative cost $c(i,j)$, which will be defined later.  

Assume that the aggregated power readings from time $t=1$ to $T$ are generated by $N$ appliances whose rated power and power deviation are known. Given the initial state of all appliances $S_0$, we can build the following tree:

\textbf{Step 1.} Set $S_0$ as the root of the tree.

\textbf{Step 2.} For each leaf node $S_{i}$ (or $S_{0}$ in the first iteration), set its children as all possible states that can be transited from $S_i$. As a result, we can add $M$ children to $S_i$, where $M =2^{N}$.

\textbf{Step 3.} Set the edge cost between $S_{i}$ and its child $S_{j}$ as
\begin{equation*}
c(i,j) = \left\{ \begin{array}{ll} 
\left\|S_j - S_i\right\|_{1} \text{ , } \text{ if }S_j \text{ satisfies } (\ref{eqt:Constraints})\\
\infty \text{ , } \text{ else. }
\end{array} \right.
\end{equation*}

\textbf{Step 4.} Repeat Step 2 to Step 4 from $t=1$ to $T$. At the end, we construct a tree $\mathcal{T} (M, T)$, where $M = 2^{N}$.

Thus, we can translate SSER into the problem of finding the minimum-cost path in $\mathcal{T}(M, T)$ from the root to a leaf (we call such path a \textit{full path} in the following). We reduce the optimization problem to its decision version.

\begin{definition}
\textbf{Decision version of SSER (d-SSER)}: Given a constant $k$, find out whether or not there exists a full path in $\mathcal{T}(M,T)$ with total cost no larger than a constant $k$. 
\end{definition}
The d-SSER can be re-formulated as
\begin{equation*}
\begin{split}
 \textbf{d-SSER} = \{ \left\langle \mathcal{T}, c, k \right\rangle:
& \mathcal{T}(M, T),\\
& c \text{ is the cost function },\\
& k \in \Re^{+}, \text{ and }\\
& \mathcal{T} \text{ has a full path with cost } \leq k\}.
\end{split}
\end{equation*}

We next reduce a well-known NP-complete problem, the Traveling Salesman Problem (TSP) to d-SSER. TSP can be formulated as
\begin{equation*}
\begin{split}
 \textbf{TSP} = \{ \left\langle \mathcal{G}, c', k \right\rangle:
& \mathcal{G} = (V, E) \text{ is a complete graph },\\
& c' \text{ is the cost function },\\
& k \in \Re^{+}, \text{ and }\\
& \mathcal{G} \text{ has a Hamiltonian cycle with cost}\\
& \leq k\}.
\end{split}
\end{equation*}

We complete the proof in two steps: firstly we show that d-SSER is NP; then, we prove that d-SSER is NP-complete by showing TSP $\leq_{P}$ d-SSER, \textit{i.e.}, there exists a reduction from TSP to d-SSER.

\subsection{d-SSER is NP}\label{ASS_1}

\begin{itemize}
\item \emph{Certificate:} A path of $\mathcal{T}(M,T)$.
\item \emph{Algorithm:} 1) Check that the path is full, \textit{i.e.}, the path starts from the root and ends at a leaf. 2) Calculate the total edge costs along the path and check if it is no larger than $k$.
\item \emph{Polynomial Time:} We need $T+1$ steps to check the path and obtain its total cost.
\end{itemize}

\subsection{d-SSER is NP-Complete}\label{ASS_2}

\begin{itemize}
\item Firstly, we develop an algorithm $\mathcal{F}:\left\langle \mathcal{G}, c', k \right\rangle \rightarrow \left\langle \mathcal{T}, c, k \right\rangle$, \textit{i.e.}, $\mathcal{G}$ and $c'$ in TSP can be transferred to $\mathcal{T}$ and $c$ in d-SSER as follow:

\textbf{Step 1.} Choose any node of $\mathcal{G}$ as the root of $\mathcal{T}$;

\textbf{Step 2.} For each leaf node of the current tree, add its children as all the other nodes of $\mathcal{G}$. Since $\mathcal{G}$ is a complete graph, we can add $\left|V\right|-1$ children to each leaf node, where $\left|V\right|$ is the number of nodes in $\mathcal{G}$.

\textbf{Step 3.} Repeat Step 2 for $\left|V\right|$ times. At the end, we build $\mathcal{T} (M, T)$, where $M = \left|V\right|-1$ and $T =\left|V\right|$. 

\textbf{Step 4.} Set the cost of edge $(i, j)$ in $\mathcal{T}$, $c(i,j)$, as follows: 
\begin{enumerate}
\item Initialization: $c(i,j) = c'(i,j)$, where $c'(i, j)$ is the edge cost in $\mathcal{G}$.
\item For each edge $(i,j)$ of $\mathcal{T}$ where $j$ is a non-leaf node, if $j$ has appeared in the path from the root (including the root) to $i$, \textit{i.e.}, $j$ is an ancestor of $i$ in the tree already, set $c(i, j)=\infty$.   
\item For each edge $(i,j)$ of $\mathcal{T}$ where $j$ is a leaf node, if $j$ is not the same as the root node, set $c(i, j)=\infty$.
\end{enumerate} 

To help understand the construction of $\mathcal{T}$ with $\mathcal{G}$, Fig.~\ref{fig:example} shows an example with three nodes in $\mathcal{G}$. 

\begin{figure}[!ht]
\centering
\includegraphics[width=3in,height=1in]{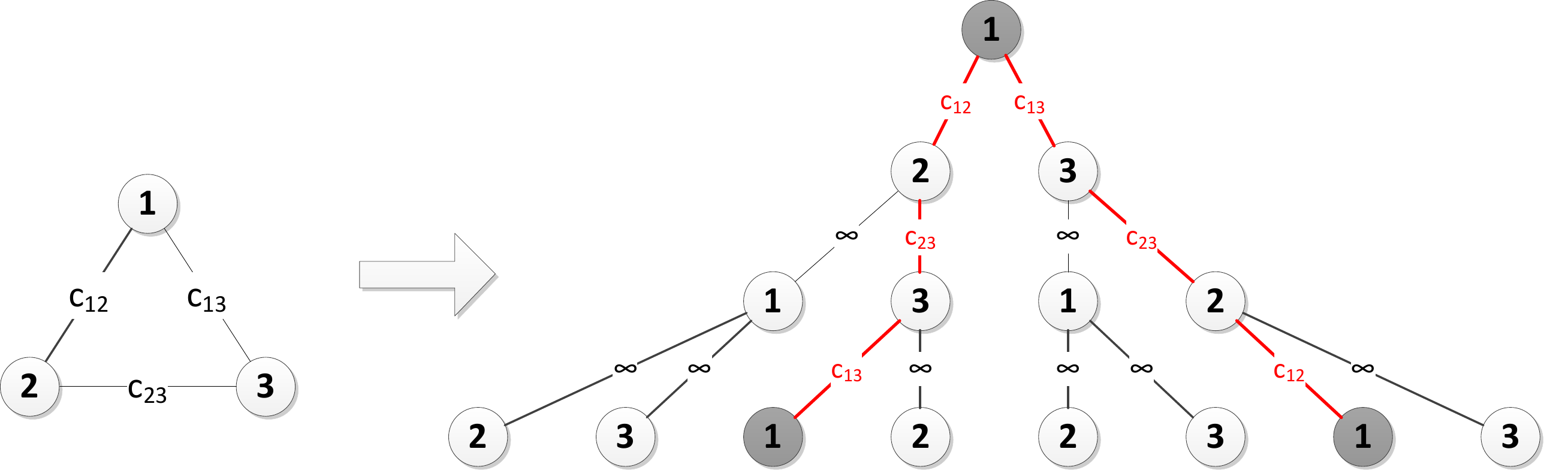}
\caption{An example showing the construction of $\mathcal{T}$ with $\mathcal{G}$}\label{fig:example}
\end{figure}

\item Secondly, it is easy to see that $\mathcal{F}$ takes $O(T^2)$ running time. 

\item Thirdly, we show that 

\begin{equation*}
\left\langle \mathcal{G}, c', k \right\rangle \in \text{TSP} \Leftrightarrow \left\langle \mathcal{T}, c, k \right\rangle \in \text{d-SSER}.
\end{equation*}

\begin{itemize}
\item $(\Rightarrow)$
\end{itemize}
\begin{equation*}
\begin{split}
    &\mathcal{G} \text{ has a Hamiltonian cycle with cost} \leq k.\\
		&\Rightarrow \text{there exists a full path in } \mathcal{T} \text{ with cost } \leq k. \\ 
		&\text{(Note that there will be no internal node along} \\ 
		&\text{the path occurring more than once, otherwise}\\
		&\text{the cost will be infinite based on the rules in Step 4.) }\\
\end{split}
\end{equation*}
\begin{itemize}
\item $(\Leftarrow)$
\end{itemize}
\begin{equation*}
\begin{split}
    &\mathcal{T} \text{ has a full path with cost } \leq k.\\
		&\Rightarrow \text{there exists a traverse instance in its}\\
		&\text{corresponding graph } \mathcal{G} \text{ with cost } \leq k.\\
		&\text{(Note that based on the tree construction,}\\
		&\text{only the full paths starting and ending at the}\\
		& \text{same node can have a cost no larger than } k,\\
		&\text{because other paths have a cost of infinity.)}\\
		&\Rightarrow \text{so } \mathcal{G} \text{ has a Hamiltonian cycle with cost } \leq k.
\end{split}
\end{equation*}

\end{itemize}

With above facts, we have proved that d-SSER is NP-complete. Since SSER problem is no easier than d-SSER, the former is NP-hard.

\section{Computational Complexity}\label{Appendix_B}

For the problem in~(\ref{eqt:Minimization_1}), if we want to find the minimum total variation with a brute-fore method, we have to traverse all possible solutions for state matrix $S$. On each time instant, there are $2^{N}$ possible combinations of states for $N$ appliances. Therefore, from time $t=1$ to $T$, the total number of feasible solutions is up to $\left(2^{N}\right)^{T}$. Thus, the computational complexity of brute-fore method to problem~(\ref{eqt:Minimization_1}) is $O(2^{N\cdot{T}})$, which is exponential.

As to the optimization problem in~(\ref{eqt:Minimization_2}), the whole searching space is limited to $k$ small local windows, and the optimization is confined within the local windows. Assume that the longest active epoch has a size of $\ell$. The searching space is $2^{N\cdot{\ell}}$. Since from $t=1$ to $T$ the total number of active epochs is $k$, the total computational complexity is upper bounded by $k\cdot{2^{N\cdot{\ell}}}$. Considering that $\ell$ is usually much smaller than $T$, the computational complexity of the original problem is  cut down significantly.
\balancecolumns

\end{document}